\documentclass[twoside,11pt]{article}

\usepackage{jmlr2e}

\usepackage{amsmath}
\usepackage{amssymb}
\usepackage{color}
\usepackage{epstopdf}

\allowdisplaybreaks
\usepackage{algorithm}
\usepackage{algorithmic}

\ShortHeadings{Top-$K$ Ranking from Pairwise Comparisons: When Spectral Ranking is Optimal}{Jang, Kim, Suh and Oh}
\firstpageno{1}

\begin{document}

\title{Top-$K$ Ranking from Pairwise Comparisons: When Spectral Ranking is Optimal}

\author{\name Minje Jang \email jmj427@kaist.ac.kr \\
       \addr Electrical Engineering, KAIST\\
       \AND
       \name Sunghyun Kim \email koishkim@etri.re.kr \\
       \addr Electronics and Telecommunications Research Institute\\
       Daejeon, Korea
       \AND
       \name Changho Suh \email chsuh@kaist.ac.kr \\
       \addr Electrical Engineering, KAIST\\
       \AND
       \name Sewoong Oh \email swoh@illinois.edu \\
       \addr Industrial and Enterprise Systems Engineering, UIUC\\
}

\editor{}

\maketitle

\begin{abstract}
We explore the top-$K$ rank aggregation problem. Suppose a collection of items is compared in pairs repeatedly, and we aim to recover a consistent ordering that focuses on the top-$K$ ranked items based on partially revealed preference information. We investigate the Bradley-Terry-Luce model in which one ranks items according to their perceived utilities modeled as noisy observations of their underlying true utilities. Our main contributions are two-fold. First, in a general comparison model where item pairs to compare are given a priori, we attain an upper and lower bound on the sample size for reliable recovery of the top-$K$ ranked items. Second, more importantly, extending the result to a random comparison model where item pairs to compare are chosen independently with some probability, we show that in slightly restricted regimes, the gap between the derived bounds reduces to a constant factor, hence reveals that a spectral method can achieve the minimax optimality on the (order-wise) sample size required for top-$K$ ranking. That is to say, we demonstrate a spectral method alone to be sufficient to achieve the optimality and advantageous in terms of computational complexity, as it does not require an additional stage of maximum likelihood estimation that a state-of-the-art scheme employs to achieve the optimality. We corroborate our main results by numerical experiments.
\end{abstract}

\begin{keywords}
Bradley-Terry-Luce models, Optimal sample complexity, Pairwise measurements, Spectral methods, Top-$K$ ranking.
\end{keywords}

\section{Introduction}
\label{sec:introduction}
Rank aggregation has been investigated in a variety of contexts such as social choice \citep{Caplin91, Soufiani14}, web search and information retrieval \citep{Dwork01}, recommendation systems \citep{Baltrunas10}, and crowd sourcing \citep{Chen13}, to name a few. The task aims to bring a consistent ordering to a collection of items, given only partial preference information.

Due to its broad range of applications, a sheer volume of work on ranking has been done. Of numerous ranking schemes developed in the literature, arguably most dominant paradigms are spectral ranking algorithms \citep{Brin98, Dwork01, Negahban12, See49, Wei52, Vig09} and maximum likelihood estimation (MLE) \citep{Ford57, Hunter04}. Postulating the existence of underlying real-valued true preferences of the items, these paradigms intend to produce preference estimates that are consistent in a global sense, usually measured by $\ell_2$ estimation error, to order the items. While it can be understood that such estimates are faithful globally with respect to the latent preferences, it is not necessarily guaranteed that they result in optimal ranking accuracy. Accurate ranking has more to do with how well the ordering of the estimates matches that of the true preferences, and less to do with how close the estimates are to the true preferences minimizing overall estimation error.

In many realistic applications of interest, however, what we expect from accurate ranking is not an ordering that respects the entire item preferences in a global sense. Instead, we expect an ordering that precisely separates only a few items that have the highest ranks from the rest. In light of this, recent work \citep{Chen15} investigated top-$K$ identification which aims to recover the correct set of top-ranked items only. As a result, it characterized the minimax limit on the sample size (i.e., sample complexity) under a long-lasting prominent statistical model, namely the Bradley-Terry-Luce (BTL) model \citep{BT52, Luce59}. In achieving the fundamental limit, its proposed scheme called {\it Spectral MLE} merges the two popular paradigms in series so as to yield low $\ell_\infty$ estimates, shown therein to be crucial in identifying top-ranked items with respect to their preferences. To start with, {\it Spectral MLE} first obtains preference estimates via a spectral method, particularly {\it Rank Centrality} \citep{Negahban12}, which produces estimates with low squared loss. And by performing additional point-wise MLEs on the estimates, it makes them have low $\ell_\infty$ estimation error, leading to successful top-$K$ ranking.

Analyzing $\ell_\infty$ error bounds can be interesting, as we can see in \citep{Chen15} where it has led to characterizing the minimax limit on the sample size for top-$K$ ranking. What makes it even more appealing is its technical challenge. Even after decades of research since the introduction of spectral methods and MLE, two dominating approaches in the literature, yet we lack notable results for $\ell_\infty$ error bounds. Analytical techniques that have proven useful to obtain tight $\ell_2$ error bounds do not translate well into obtaining meaningful $\ell_\infty$ error bounds. There lie our main contributions: tangible progress in $\ell_\infty$ analysis.

{\bf Main contributions.} In this work, we provide a tight analysis of $\ell_\infty$ error bounds of a spectral method, making progress toward richer understanding of rank aggregation. The analysis makes it possible for us to characterize conditions under which the spectral method achieves the minimax optimal performance, by comparing it with a fundamental bound  that delineates the performance limit beyond which any ranking algorithm cannot achieve. To be more concrete, we investigate reliable recovery of top-$K$
ranked items under the BTL pairwise comparison model in which one ranks items according to their perceived utilities modeled as noisy observations of their underlying true utilities. We consider mainly two comparison models: one is a deterministic model in which item pairs we compare are given a priori; the other is a random model in which item pairs we compare are chosen in a random and non-adaptive manner. As our main results, in the former model, we derive an upper and lower bound on the sample size for reliable recovery of top-$K$ ranked items (Theorems~\ref{thm:generalgraph} and~\ref{thm:generalconverse}), which respectively correspond to sufficient and necessary conditions for reliable top-$K$ identification, when a spectral method \emph{Rank Centrality} is employed. Inspecting the gap between the derived bounds allows us to identify conditions under which \emph{Rank Centrality} can be optimal. We observe that, for well-balanced cases where the number of distinct items that an item is compared to (which we call degree) does not deviate greatly from its minimum to its maximum, how the gap scales can be nicely expressed in terms of degree (details in Section~\ref{sec:mainresults} after Theorem~\ref{thm:generalconverse}). In the random model we consider, item pairs we compare are first chosen independently with probability $p$ and the chosen pairs are repeatedly compared (hence random and non-adaptive). Finding the random model fit for the well-balanced case, we extend the aforementioned results and get a stronger one. We demonstrate that the gap shrinks to the order of constant (Theorem~\ref{thm:ergraph}), hence show that a spectral method alone can achieve the order-wise optimal sample complexity for top-$K$ ranking that has recently been characterized under the same model in \citep{Chen15}. There are two distinctions to note in comparison with the results in \citep{Chen15}. First, we show that a spectral method can achieve the limit in so-called dense regimes where the number of distinct item pairs we compare is somewhat large. That is, in comparison to the regimes in which Chen and Suh characterized the limit, the regimes in which we achieve it are slightly restricted. Second, we show that applying only \emph{Rank Centrality} is sufficient to achieve the limit in the regimes mentioned earlier, hence is more advantageous in terms of computational complexity in comparison to {\it Spectral MLE} that merges a spectral method (particularly \emph{Rank Centrality} in \citep{Negahban12}) and an additional stage performing coordinate-wise MLEs. 

{\bf Related work.} Perhaps most relevant are \citep{Chen15} and \citep{Negahban12}. To the best of our knowledge, Chen and Suh focused on top-$K$ identification under the random comparison model of our interest for the first time. A key distinction with our work is that while we employ only a spectral method to obtain bounds on $\ell_\infty$ estimation error, they incorporated an additional refinement stage that performs successive point-wise MLEs. 
Negahban et al. developed {\it Rank Centrality} on which our proposed ranking scheme is solely based. 
Perhaps surprisingly, it was proved that {\it Rank Centrality}, 
a simple spectral method, achieves the same performance as MLE in $\ell_2$ error. 
A priori, there is no reason to believe that a spectral method can achieve such a strong minimax optimal performance.
In a similar spirit, we show that this spectral method is also minimax optimal in $\ell_\infty$ error, 
achieving the same optimality guarantee as the MLE based algorithm in \citep{Chen15}. 
The main objective of our work is in identifying the regimes where spectral methods are as good as MLE, 
and proving the minimax optimality of {\it Rank Centrality} in those regimes.


\citet{MG15} recently developed an algorithm that also shares a spirit of spectral ranking, called \emph{Iterative Luce Spectral Ranking (I-LSR)}, and showed its performance to be the same as MLE for underlying preference scores. \citet{Rajkumar14} put forth statistical assumptions that ensure the convergence of several rank aggregation methods including \emph{Rank Centrality} and MLE to an optimal ranking. They derived sample complexity bounds, although the statistical optimality is not rigorously justified and total ordering instead of top-$K$ ranking is concerned. In the pairwise preference setting, many works with different interests from ours have been done. Some studied \emph{active} ranking where samples are obtained adaptively. \citet{JN11} considered perfect total ranking and characterized the query complexity gain of adaptive sampling in the noise-free case, and the works of \citep{Braverman08, JN11, Ailon12, Wauthier13} explored the query complexity in the presence of noise while aiming at approximate total rankings. \citet{Eriksson13} proposed a scheme that intends to find top-$K$ queries when observation errors are assumed to be i.i.d. Some works looked into models different from the BTL model. The works of \citep{Lu11, Busa14} considered ranking problems with pairwise comparison data under the Mallows model \citep{Mallows57}. \citet{Soufiani13} broke full rankings into pairwise comparisons toward parameter estimation under the Plackett-Luce (PL) model \citep{Plackett75}. \citet{Hajek14}, under the PL model, derived minimax lower bounds of parameter estimation error when schemes that break partial rankings into pairwise comparisons are used.

Very recently, \citet{Shah15} showed that a simple counting method \citep{Borda781} can achieve the fundamental limit on the sample size, up to constant factors, for top-$K$ ranking under a general parametric model in which observations depend only on the predefined probabilities of one item preferred to another \citep{Shah152}, including the BTL model as a special case. However, their assumption that the number of comparisons for each item pair follows a Binomial distribution led to a nearly complete observation model where almost every item pair is compared at least once. On the contrary, we examine a different observation model (which we will describe in detail soon), also considered in \citep{Negahban12} and \citep{Chen15}, which well captures the comparison graph structure that affects the sample complexity (see Theorem 1 of \citep{Negahban12} and numerical experiments in Section~\ref{sec:simulations}).


{\bf Notation.} Unless specified otherwise, we use $[n]$ to represent $\{ 1, 2, \dots, n \}$, and $\mathcal{G}_{n,p}$ to represent an Erd{\H o}s-R{\'e}nyi random graph where total $n$ vertices reside and each pair of vertices is connected by an edge independently with probability $p$, and $d_{i}$ to represent the out-degrees of vertex $i$.
%
%
%
%

 

\section{Problem Formulation} 
\label{sec:model}

{\bf Comparison model and assumptions.}
Suppose we perform a few pairwise evaluations on $n$ items. To gain a statistical understanding toward the ranking limits, we assume the pairwise comparison outcomes are generated based on the Bradley-Terry-Luce (BTL) model \citep{BT52, Luce59}, a long-established model that has been studied in numerous applications \citep{Agresti14, Hunter04}.

\begin{itemize}
\item \textit{Preference scores.} The BTL model postulates the existence of an underlying preference vector $\boldsymbol{w} := \{ w_1, w_2, \dots, w_n \}$, where $w_i$ represents the preference score of item $i$. The outcome of each pairwise comparison depends solely on the latent scores of the items being compared. Without loss of generality, we assume that
\begin{align}
w_1 \geq w_2 \geq \cdots \geq w_n > 0.
\end{align}
We assume that the range of the scores is fixed irrespective of $n$. For some positive constants $w_{\rm min}$ and $w_{\rm max}$:
\begin{align}
w_i \in [w_{\rm min}, w_{\rm max}], \quad 1 \leq i \leq n,
\end{align}
In fact, the case in which the range $\frac{w_{\rm max}}{w_{\rm min}}$ grows with $n$ can be translated into the above fixed-range regime by separating out those items with vanishing scores (e.g. via a voting method like Borda count \citep{Borda781, Ammar11}).

\item \textit{Comparison model.} We denote by $\mathcal{G} = ([n], \mathcal{E})$ a comparison graph in which items $i$ and $j$ are compared if and only if $(i,j)$ belongs to the edge set $\mathcal{E}$. We take into account two kinds of comparison graphs. We examine general comparison graphs which can exhibit all possible topologies described by an edge set $\mathcal{E}$ given a vertice set $[n]$. Furthermore, we investigate random comparison graphs constructed by the Erd{\H o}s-R{\'e}nyi random graph model in which each pair of vertices is connected by an edge independently with probability $p$.

\item \textit{Pairwise comparisons.} For each $(i,j) \in \mathcal{E}$, we observe $L$ comparisons between items $i$ and $j$. The outcome of the $\ell^{\rm th}$ comparison between them, denoted by $y_{ij}^{(\ell)}$, is generated based on the BTL model:
\begin{align}
y_{ij}^{(\ell)} = \left\{  \begin{array}{l} 1 \quad \text{with probability}~\frac{w_i}{w_i+w_j} \\ 0 \quad \text{otherwise,} \end{array} \right.
\end{align}
where $y_{ij}^{(\ell)} = 1$ indicates that item $i$ is preferred over item $j$. We adopt the convention $y_{ji}^{(\ell)} = 1 - y_{ij}^{(\ell)}$. We assume that conditional on $\mathcal{G}$, $y_{ij}^{(\ell)}$'s are jointly independent over all $\ell$ and $i<j$. For ease of presentation, we represent the collection of sufficient statistics as
\begin{align} \label{def: observation}
\boldsymbol{y} := \{ y_{ij} : i < j, \in \mathcal{E} \}, ~y_{ij} := \frac{1}{L} \sum_{\ell=1}^{L} y_{ij}^{(\ell)}.
\end{align}
\end{itemize}

{\bf Performance metric and goal.}
Given the pairwise comparisons, one wishes to know whether or not the top-$K$ ranked items are identifiable. In light of this, we consider the probability of error $P_{\sf e}$ in identifying the correct \emph{set} of the top-$K$ ranked items, namely,
\begin{align}
P_{e}(\psi) := \mathbb{P} \left\{ \psi(\boldsymbol{y}) \neq [K] \right\},
\end{align} 
where $\psi$ is any ranking scheme that returns a set of $K$ indices and $[K]$ is the set of the first $K$ indices. Our goal in this work is to characterize the \emph{admissible region} ${\cal R}_{\boldsymbol{w}}$ of $L$ in which top-$K$ ranking is feasible for a given BTL parameter $\boldsymbol{w}$, in other words, $P_{e}$ can be vanishingly small as $n$ grows. The admissible region $\cal{R}_{\boldsymbol{w}}$ is defined as
\begin{align} \label{def:region}
\mathcal{R}_{\boldsymbol{w}} := \left\{ L : \lim_{n \to \infty} P_{e}(\psi(\boldsymbol{y})) = 0 \right\}.
\end{align}

For a comparison graph $\mathcal{G} = ([n], \mathcal{E})$, we are interested in the \emph{sample complexity} defined as,
\begin{align} \label{def:complexity}
S_{\delta} := \underset{L \in \mathbb{Z}^+}{\min} \underset{\boldsymbol{a} \in \Omega_\delta}{\sup} \left\{ L|\mathcal{E}| : L \in {\cal R}_{\boldsymbol{a}} \right\},
\end{align}
where $\Omega_\delta = \{ \boldsymbol{a} \in \mathbb{R}^n : (a_K - a_{K+1} )/ a_{\rm max} \geq \delta \}$. Note that the way the sample complexity is defined as (\ref{def:complexity}) shows that we investigate minimax scenarios in which nature may behave in an adversarial manner with the worst-case preference scores $\boldsymbol{w}$.



\section{Main Results}
\label{sec:mainresults}
The most crucial part of top-$K$ ranking hinges on separating the two items near the decision boundary, i.e., the $K^{\rm th}$ and $(K+1)^{\rm th}$ ranked items. Unless the gap is large enough, noise in the observations can lead to erroneous estimates. In view of this, we pinpoint a separation measure as
\begin{align}
\Delta_K := \frac{w_K - w_{K+1}}{w_{\rm max}}.
\end{align}
This measure turns out to play a key role in determining the fundamental limits of top-$K$ identification.

As noted in \citep{Ford57}, if the comparison graph $\mathcal{G}$ is not connected, then it is impossible to determine the relative preferences between two disconnected components. Hence, we assume all comparison graphs considered in this paper are connected. For Erd{\H o}s-R{\'e}nyi model, we make the following assumption for the connectivity:
\begin{align}
p > \frac{\log n}{n}.
\end{align}

Our main findings are sufficient and necessary conditions derived for reliable top-$K$ identification for a general comparison graph. Especially, for a random comparison graph according to the Erd{\H o}s-R{\'e}nyi model, we can attain an order-wise tight sufficient condition for feasible top-$K$ ranking. We first state our results for general comparison graphs.
\begin{theorem} \label{thm:generalgraph}
Given a comparison graph $\mathcal{G} = ([n],\mathcal{E})$ and $L \geq \left\lceil c_{1}\frac{\log n}{d_{\rm max}}\left(\frac{d_{\rm max}}{\gamma d_{\rm min}}\right)^2 \right\rceil$, if
\begin{align} \label{ieq:l_general}
L|\mathcal{E}| \geq \left(c_2 + c_3\frac{\sqrt{n}d_{\max}}{\gamma d_{\min}}\|\mathcal{L}^2\|_{2,\infty} \right)^2\frac{|\mathcal{E}|\log n}{d_{\max} \Delta_K^2},
\end{align}
then \emph{Rank Centrality} correctly identifies the top-$K$ ranked items with probability at least $1 - 2n^{-2}$, where $c_{1}$, $c_{2}$ and $c_{3}$ are some numerical constants, $\mathcal{L}$ is the Laplacian matrix of graph $\mathcal{G}$ whose entries are defined as $\mathcal{L}_{ij} := \frac{1}{d_{i}}\mathbb{I}[(i,j) \in \mathcal{E}]$, and $\|A\|_{2,\infty} := \max\limits_{j}\left(\sqrt{\sum\limits_{i} |A_{ij}|^2}\right)$. Here $\gamma$ is the spectral gap of matrix $\mathcal{L}$ defined as the difference between the two largest absolute eigenvalues of $\mathcal{L}$, $d_{\rm max}$ is the maximum out-degree of vertices in $\mathcal{E}$ and $d_{\rm min}$ is the minimum.
\end{theorem}
Note that in terms of the sample complexity defined as (\ref{def:complexity}), 
this theorem establishes a sufficient condition of the sample complexity for reliable top-$K$ ranking. Precisely, 
\begin{align} \label{sample achieve}
S_{\Delta_K} \lesssim \frac{|\mathcal{E}|}{d_{\max}}\left(1 + \frac{\sqrt{n}d_{\max}}{\gamma d_{\min}}\|\mathcal{L}^2\|_{2,\infty}\right)^2\frac{\log n}{\Delta_K^2}.
\end{align} 
We provide the proof of this theorem in Section~\ref{sec:prooftheorem}.

What follows next is a necessary condition for reliable top-$K$ ranking.
\begin{theorem} \label{thm:generalconverse}
Fix $\epsilon \in (0, \frac{1}{2})$. Given a comparison graph $\mathcal{G} = ([n], \mathcal{E})$, if
\begin{align}
L|\mathcal{E}| \leq c_4(1-\epsilon)\frac{n \log n}{\Delta_K^2},
\end{align}
for some numerical constant $c_4$, then for any ranking scheme $\psi$, there exists a preference score vector $w$ with seperation $\Delta_K$ such that $P_e (\psi) \geq \epsilon$.
\end{theorem}
This result implies that we need at least $L|\mathcal{E}| \geq c_4 \frac{n \log n}{\Delta_K^2}$ for reliable top-$K$ ranking. Then, when we express the result of Theorem \ref{thm:generalconverse} in terms of the sample complexity, that is
\begin{align} \label{sample converse}
S_{\Delta_K} \gtrsim \frac{n \log n}{\Delta_K^2}.
\end{align}
The proof is a generalized version of Theorem 2 in \citep{Chen15}. We provide the proof of this theorem in Section \ref{sec:generalconverse}.

For well-balanced cases where $d_{\rm min} = \Theta(d_{\rm max})$, one can verify that $\| \mathcal{L}^2 \|$ is on the order of $\sqrt{\frac{1}{n} + \frac{1}{d_{\rm min}^2}}$ and $\frac{|\mathcal{E}|}{d_{\rm max}}$ is on the order of $n$. Taking these two together, the gap between the necessary condition and the sufficient condition can be shown as a factor of $1 + \frac{n}{d_{\rm min}^2}$. We note that for well-balanced graphs, when $d_{\rm min}$ is at least $O(\sqrt{n})$, the gap disappears. That is, \emph{Rank Centrality} is optimal. We make this point precise in the following theorem where we analyze comparisons over Erd{\H o}s-R{\'e}nyi graphs.

\begin{theorem} \label{thm:ergraph}
Suppose $\mathcal{G} \sim \mathcal{G}_{n,p}$. There exist positive numerical constants $c_4 > 1$, $c_5$ and $c_6$ such that if $p \geq c_4 \sqrt{\frac{\log n}{n}}$, $L \geq \left \lceil c_{5}\frac{\log n}{np} \right \rceil$, and 
\begin{align} \label{er sufficient}
\frac{n^2pL}{2} \geq c_6 \frac{n \log n}{\Delta_K^2},
\end{align} 
then \emph{Rank Centrality} correctly identifies the top-$K$ ranked items with probability at least $1 - 4n^{-\alpha}$, where $\alpha := \min (1, \frac{3}{28}c_4^2)$.
\end{theorem}
This result offers a much tighter bound than Theorem \ref{thm:generalgraph}. In terms of the sample complexity, a sufficient condition of the sample complexity on a random comparison graph is
\begin{align} \label{sufficient ergraph} 
S_{\Delta_K} \lesssim \frac{|\mathcal{E}|\log n}{np\Delta_K^2} \asymp \frac{n \log n}{\Delta_K^2},
\end{align} 
since the sufficient condition for reliable ranking on a random comparison graph given by (\ref{er sufficient}) is $L|\mathcal{E}| \gtrsim \frac{|\mathcal{E}|\log n}{np\Delta_K^2}$, and the number of item pairs being compared $|\mathcal{E}|$ concentrates to $\frac{n^2p}{2}$ for Erd{\H o}s-R{\'e}nyi random comparison graphs.
Note that this sufficient condition for reliable top-$K$ ranking matches the necessary condition in (\ref{sample converse}). That is, for random comparison graphs that follow the Erd{\H o}s-R{\'e}nyi model, we can establish the minimax optimality of \emph{Rank Centrality}. Precisely,
\begin{align} \label{minsamplecomplexity}
S_{\Delta_K} \asymp \frac{ n \log n }{ \Delta_K^2 }.
\end{align}
We provide the proof of this theorem in Section~\ref{sec:ergraph_proof}.

\begin{figure}[ht]
\begin{center}
\centerline{\includegraphics[width=0.45\columnwidth]{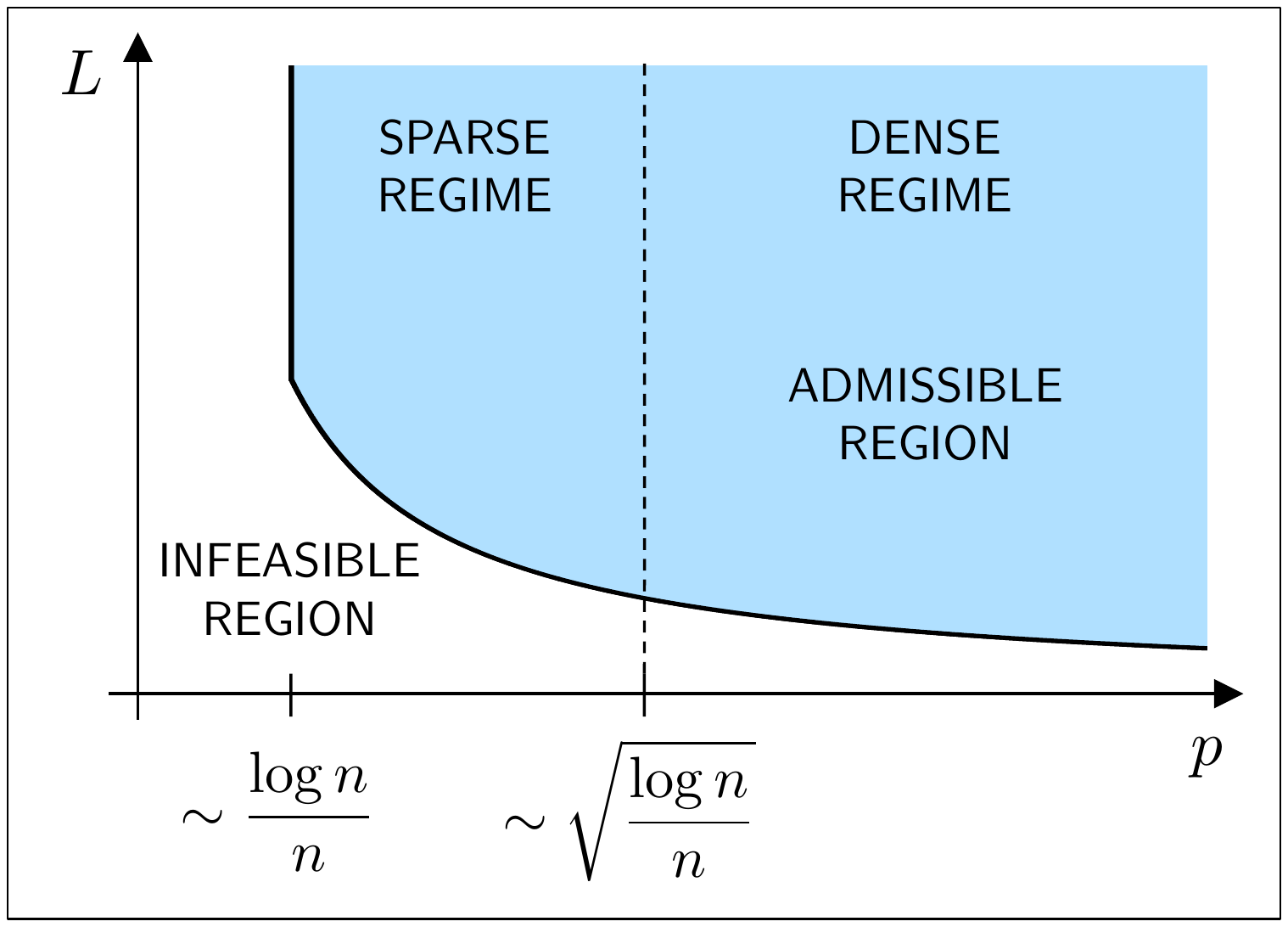}}
\vskip -0.1in
\caption{{\it Spectral MLE}, which merges {\it Rank Centrality} and an additional refinement stage that performs coordinate-wise MLEs, achieves reliable top-$K$ ranking in the entire admissible region depicted above. Our analysis reveals that in the dense regime {\it Rank Centrality} alone suffices to achieve it.}
\label{fig:admissible}
\end{center}
\vskip -0.2in
\end{figure}

Our main contribution is the establishment of a sufficient condition for top-$K$ identification, which matches the necessary condition derived in \citep{Chen15} for random comparison graphs constructed by the Erd{\H o}s-R{\' e}nyi model. It is important to point out two notable distinctions in achievability compared to \citep{Chen15}. First, a spectral method such as {\it Rank Centrality} \citep{Negahban12} suffices to achieve the order-wise tight sample complexity, without relying on an additional process of local refinement employed in \citep{Chen15}. Second, our main results concern a slightly denser regime, indicated by the condition $p \gtrsim \sqrt{\frac{\log n}{n}}$, in which many distinct item pairs are likely to be compared. As shown in \citep{Chen15}, the dense regime condition $p \gtrsim \sqrt{\frac{\log n}{n}}$ is not necessary for top-$K$ identification. However, it is not clear yet whether or not the condition is required under our approach that employs only a spectral method. Our speculation is that the sparse regime condition, indicated by $\frac{\log n}{n} \lesssim p \lesssim \sqrt{\frac{\log n}{n}}$, may not be sufficient for spectral methods to achieve reliable top-$K$ identification (to be discussed in Section~\ref{sec:simulations}).

To validate our main result based on the Erd{\H o}s-R{\'e}nyi model, we conducted numerical experiments (to be illustrated in Section~\ref{sec:simulations}). In the dense regime indicated by $p \gtrsim \sqrt{\frac{\log n}{n}}$, the experimental results clearly illustrate that {\it Rank Centrality} alone (a spectral method) achieves reliable top-$K$ identification as {\it Spectral MLE} does. In the sparse regime indicated by $\frac{\log n}{n} \lesssim p \lesssim \sqrt{\frac{\log n}{n}}$, however, {\it Rank Centrality} fails to achieve it, which leads us to the aforementioned speculation.

\begin{remark}
As mentioned earlier, our ranking algorithm is based solely on a spectral method, \emph{Rank Centrality} in \citep{Negahban12}, which enjoys nearly-linear time computational complexity. Hence, not only can the information-theoretic limit promised by (\ref{minsamplecomplexity}) be achieved by a computationally efficient low-complexity algorithm, but also we can achieve it with much less computational overhead as compared to \emph{Spectral MLE} in \citep{Chen15}, which employs an additional refinement stage.
\end{remark}


\begin{remark}
By the hypothesis $p \gtrsim \sqrt{\frac{\log n}{n}}$ in Theorem~\ref{thm:ergraph}, $n^2pL \gtrsim n \sqrt{n \log n}L$. It means that, unless $\frac{n \log n}{\Delta_K^2} \gtrsim n \sqrt{n \log n} L$, the minimax optimality of the sample complexity we claim to characterize is on the order of $n \sqrt{n \log n} L$, not $\frac{n \log n}{\Delta_K^2}$. We note that we consider a regime where $\Delta_K$ is not on the constant order, so it is reasonable to assume $\Delta_K \lesssim \left( \frac{\log n}{n L^2} \right)^{\frac{1}{4}}$, which leads to $\frac{n \log n}{\Delta_K^2} \gtrsim n \sqrt{n \log n} L$. Note that since there are $n$ items each with $w_i\in[w_{\rm min},w_{\rm max}]$, a typical regime of $\Delta_K$ scales as $1/n$. Therefore, we conclude that the minimax optimality of the sample complexity is on the order of $\frac{n \log n}{\Delta_K^2}$.
\end{remark}

\section{Experimental Results}
\label{sec:simulations}

We conduct a series of synthetic experiments to corroborate our main result in Theorem~\ref{thm:ergraph}. We consider both dense ($p \gtrsim \sqrt{\frac{\log n}{n}}$) and sparse ($\frac{\log n}{n} \lesssim p \lesssim \sqrt{\frac{\log n}{n}}$) regimes. To be more precise, we set constant $c_1 = 2$, and set $p_{\rm dense} = 0.25$ and $p_{\rm sparse} = 0.025$, to make each be in its proper range. To specify the implementation parameters, we use $n = 500$, $K = 10$, and $\Delta_K = 0.1$. Each result in all numerical simulations is obtained by averaging over 10000 Monte Carlo trials.

\begin{figure}[ht!]
\begin{center}
\includegraphics[width=0.45\columnwidth]{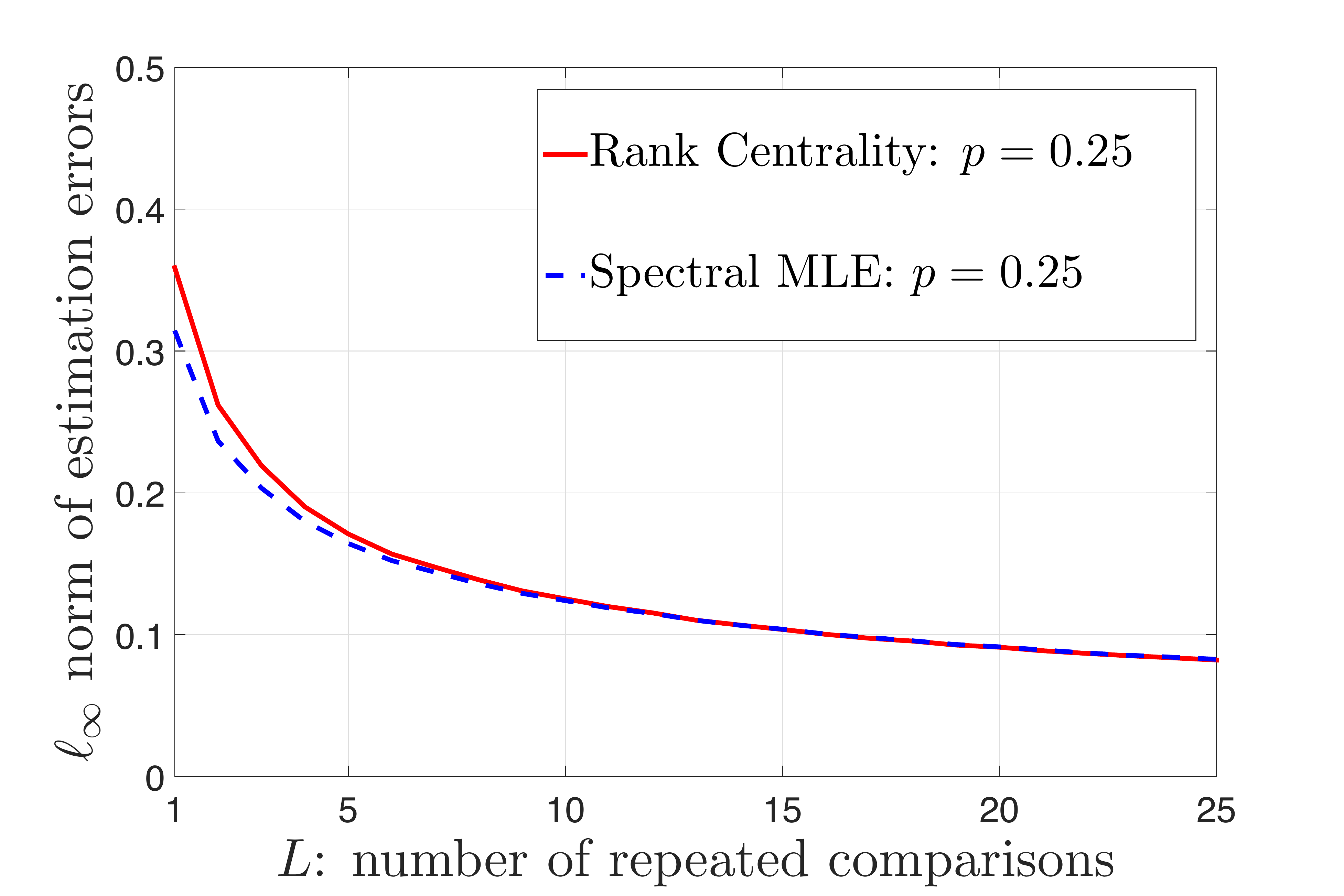}
\includegraphics[width=0.45\columnwidth]{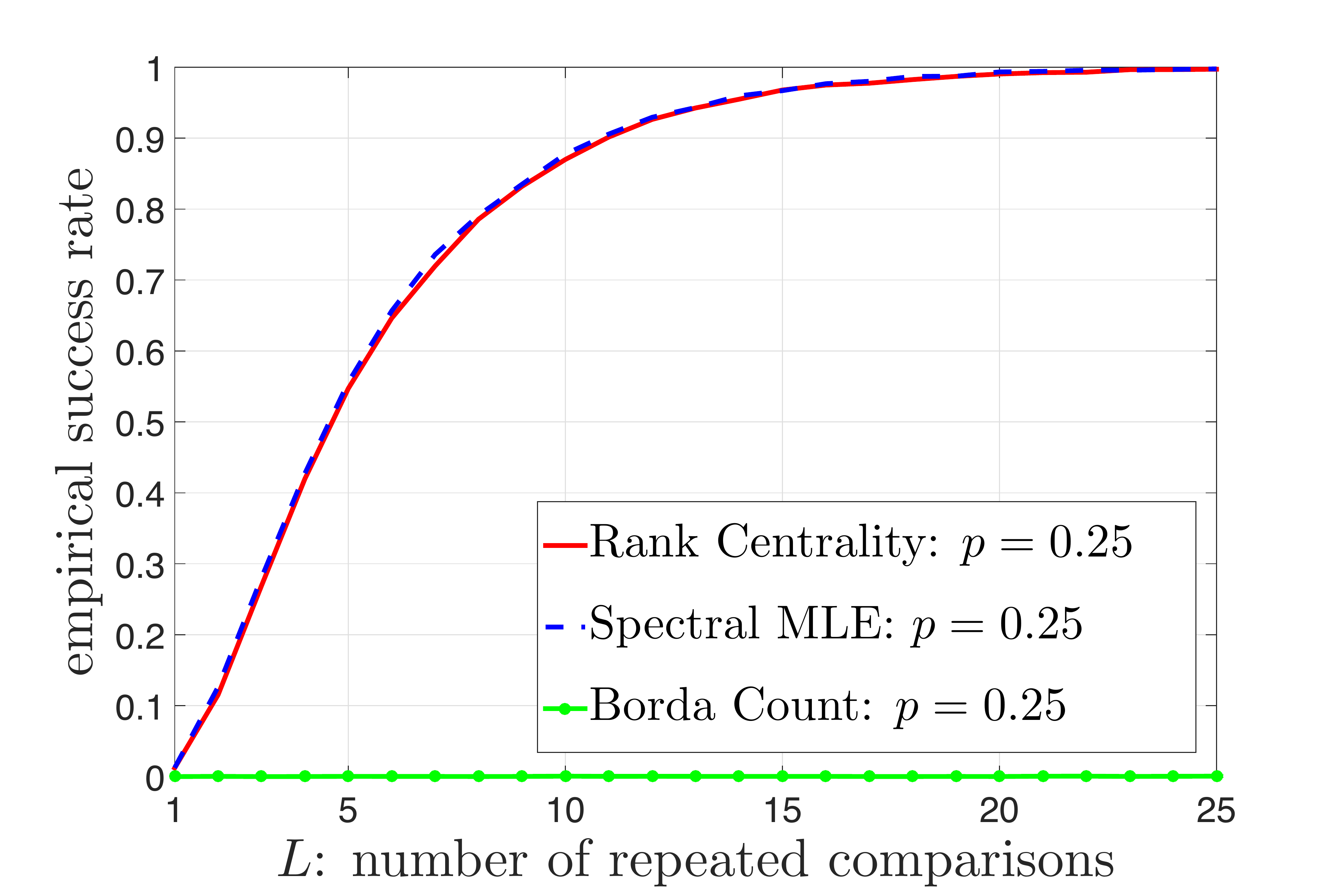}
\vskip -0.1in
\caption{Dense regime ($p_{\rm dense} = 0.25$): empirical $\ell_\infty$ estimation error v.s. $L$ (left); empirical success rate v.s. $L$ (right).}
\label{fig:dense}
\end{center}
\vskip -0.35in
\end{figure}

\begin{figure}[ht!]
\begin{center}
\includegraphics[width=0.45\columnwidth]{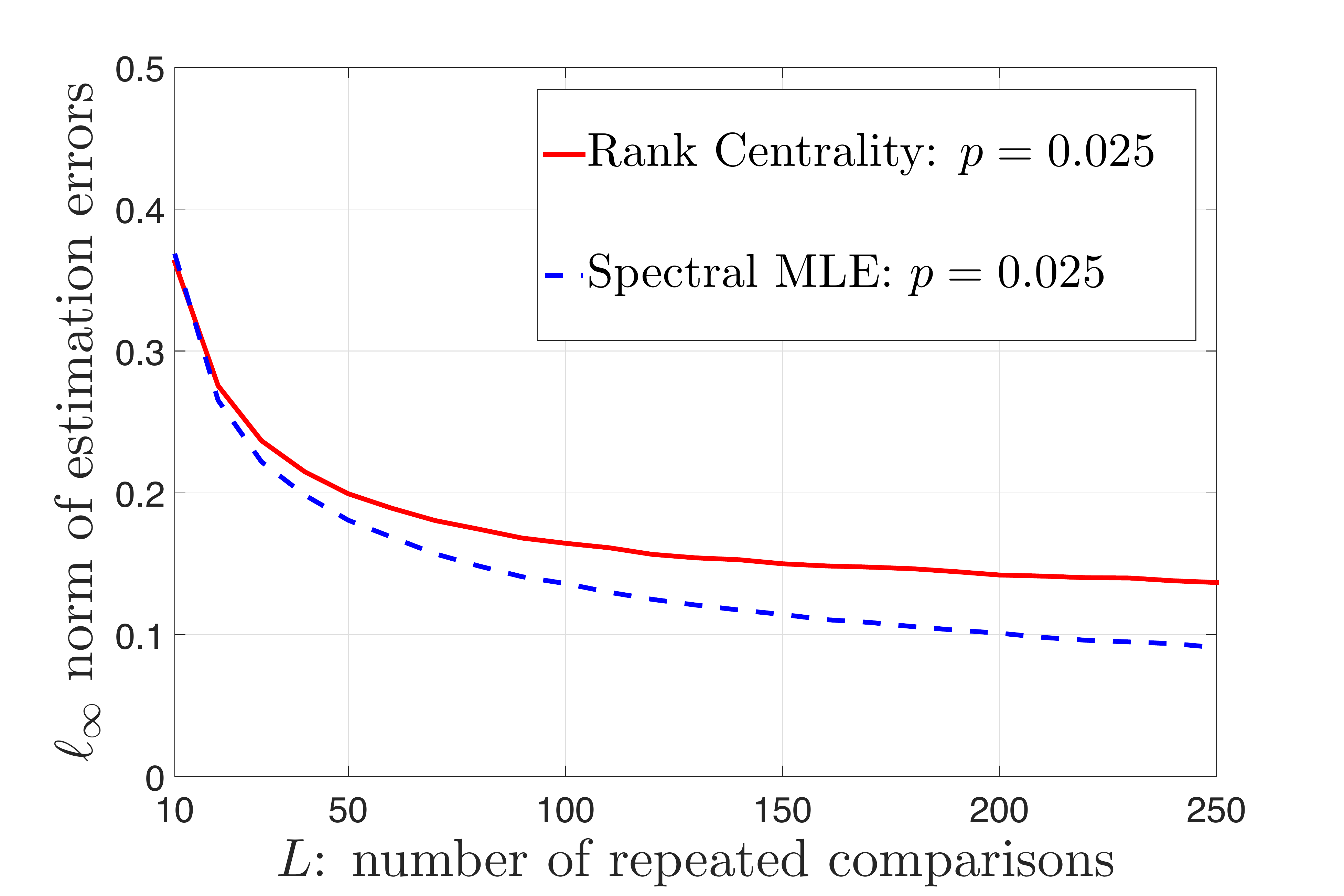}
\includegraphics[width=0.45\columnwidth]{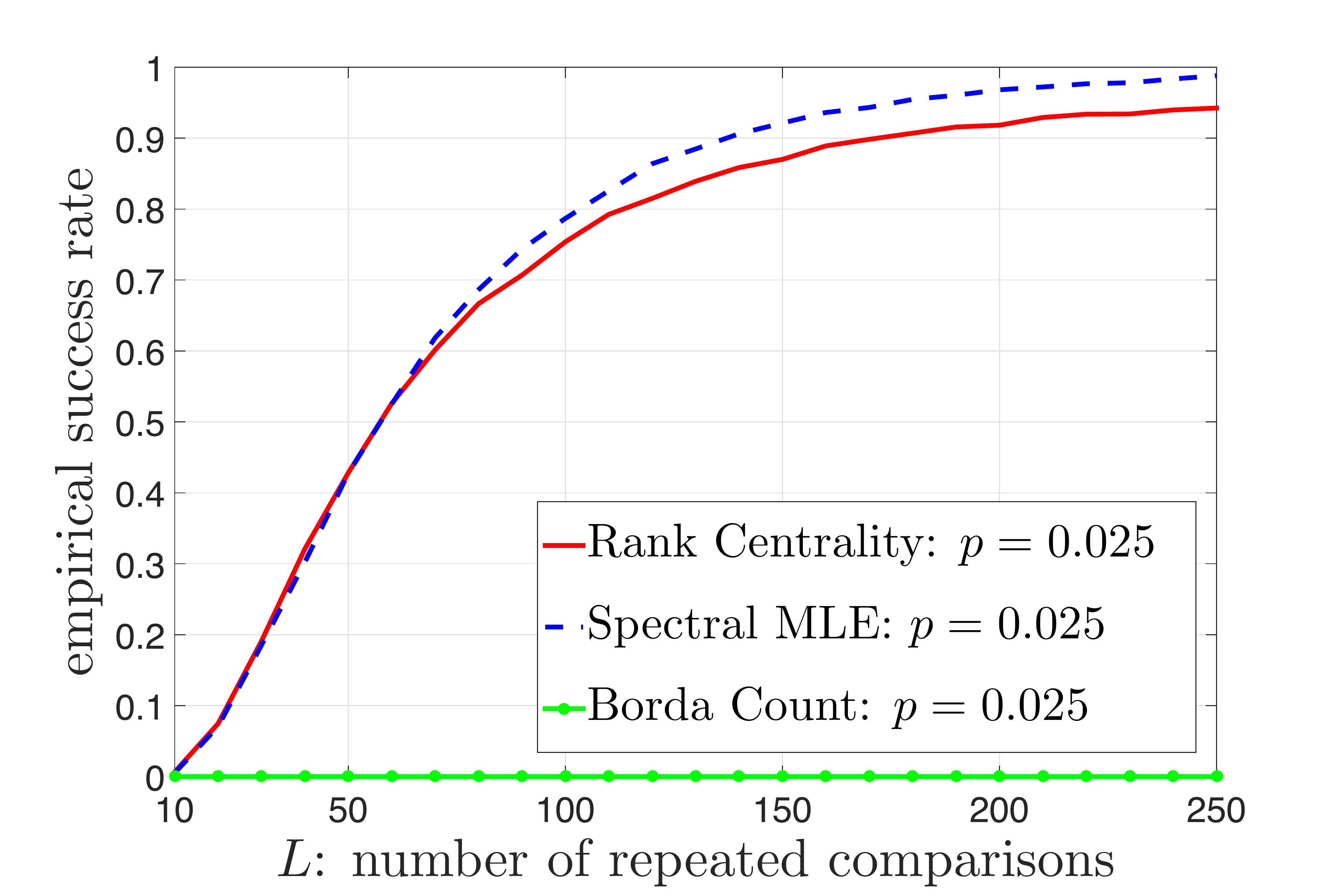}
\vskip -0.1in
\caption{Sparse regime ($p_{\rm sparse} = 0.025$): empirical $\ell_\infty$ estimation error v.s. $L$ (left); empirical success rate v.s. $L$ (right).}
\label{fig:sparse}
\end{center}
\vskip -0.35in
\end{figure}

Figure~\ref{fig:dense} illustrates the numerical experiments conducted in the dense regime. We see that as $L$ increases, meaning as we get to obtain pairwise evaluation samples beyond the minimal sample complexity, (1) the $\ell_\infty$ estimation error of \emph{Rank Centrality} decreases and soon meets that of \emph{Spectral MLE} (left); (2) the success rate of \emph{Rank Centrality} increases and soon hits $100\%$ along with \emph{Spectral MLE} (right). The curves clearly support our results; in the dense regime specified by $p \gtrsim \sqrt{\frac{\log n}{n}}$, \emph{Rank Centrality} (a spectral method) alone can achieve reliable top-$K$ ranking.

Figure~\ref{fig:sparse} illustrates the numerical experiments conducted in the sparse regime. We see that, in contrast with the experiments in the dense regime, as $L$ increases, (1) the $\ell_\infty$ estimation error of \emph{Rank Centrality} decreases but does not meet that of \emph{Spectral MLE} (left); (2) the success rate of \emph{Rank Centrality} increases but does not reach that of \emph{Spectral MLE} which hits nearly $100\%$ (right). The curves lead us to speculate that the sparse regime condition specified by $\frac{\log n}{n} \lesssim p \lesssim \sqrt{\frac{\log n}{n}}$ may not be sufficient for spectral methods to achieve reliable top-$K$ identification.

\section{Conclusion and Future Work}
\label{sec:futurework}
We investigated top-$K$ rank aggregation from pairwise data. We demonstrated that a spectral method alone, which features nearly-linear time computational complexity, is sufficient in achieving the minimal sample complexity in the dense regime. Some limitations of our results suggest future directions. Exploring if a spectral method can also achieve reliable top-$K$ identification in (part of) the sparse regime would be the most interesting one. \citet{MG15} proposed an \emph{Iterative Luce Spectral Ranking (I-LSR)} algorithm that has a spirit of spectral ranking, and showed that surprisingly the performance of \emph{I-LSR} is the same as MLE for underlying preference scores. Motivated by their results, we can set out to investigate \emph{I-LSR} to see if it can achieve the minimax optimality in the sparge regime. Analyzing spectral methods under comparison graphs not limited to the Erd{\H o}s-R{\'e}nyi graphs as well as other choice models such as the Plackett-Luce model \citep{Plackett75, Hajek14,MG15} could be another.

\section{Proof of Theorem~\ref{thm:generalgraph}}
\label{sec:prooftheorem}

\subsection{Algorithm Description}
\label{sec:algodesc}
\begin{algorithm}
   \caption{Rank Centrality \citep{Negahban12}}
   \label{alg:rankcentrality}
\begin{algorithmic}
   \STATE {\bfseries Input:} The collection of sufficient statistics $\boldsymbol{y} = \left\{ y_{ij} : (i,j) \in \mathcal{E}, y_{ij} = \frac{1}{L}\sum_{\ell=1}^L y_{ij}^{(\ell)} \right\}$.
   
	\STATE {\bfseries Compute} the transition matrix $\hat{\boldsymbol{P}} = [\hat{P}_{ij}]_{1 \leq i,j \leq n}$:
\begin{align*}
\hat{P}_{ij} = \left\{
\begin{array}{ll}
    \frac{1}{d_{\rm max}}y_{ij} & \text{if } (i,j) \in \mathcal{E}; \\
	1 - \frac{1}{d_{\rm max}} \sum_{k : (k,j) \in \mathcal{E}} y_{kj} & \text{if }i=j; \\
0 & \text{otherwise}.
\end{array}
\right.
\end{align*}

	\STATE {\bfseries Output} the stationary distribution of matrix $\hat{\boldsymbol{P}}$.
\end{algorithmic}
\end{algorithm}

In an ideal scenario where we obtain an infinite number of samples per pairwise comparison, i.e., $L \to \infty$, sufficient statistics $y_{ij}$ converge to $\frac{w_i}{w_i + w_j}$. Then, constructed matrix $\hat{\boldsymbol{P}}$ defined in Algorithm~\ref{alg:rankcentrality} becomes a matrix $\boldsymbol{P}$ whose entries $[P_{ij}]_{1 \leq i, j \leq n}$ are defined as
\begin{align} \label{idealP}
P_{ij} = \left\{
\begin{array}{ll}
    \frac{1}{d_{\rm max}}\frac{w_i}{w_i + w_j} & \text{if } (i,j) \in \mathcal{E}; \\
	1 - \frac{1}{d_{\rm max}} \sum_{k : (k,j) \in \mathcal{E}} \frac{w_k}{w_k + w_j} & \text{if }i=j; \\
0 & \text{otherwise}.
\end{array} \right.
\end{align}
The entries for observed pairs of items, $P_{ij} = \frac{1}{d_{\rm max}} \frac{w_i}{w_i + w_j}$, represent the relative likelihood of item $i$ being preferred to item $j$. Intuitively, random walks of $\boldsymbol{P}$ over the long run will visit some states (corresponding to items) more often, if they have been preferred to other frequently-visited states or preferred to many other states.

The random walks have some properties that lead us to a desirable outcome. We can see that the walks are reversible, as $w_i P_{ji} = w_j P_{ij}$ holds, thus have a stationary distribution equal to the preference score vector $\boldsymbol{w} = \{w_1, \dots, w_n \}$, up to some constant scaling. We can also see that under the assumption that guarantees connectivity, the walks are irreducible, thus the stationary distribution is unique. To find the stationary distribution of the walks of $\boldsymbol{P}$ is to retrieve the precise underlying preference scores.


It is clear that random walks of $\hat{\boldsymbol{P}}$, which can be viewed as a noisy version of $\boldsymbol{P}$, will give us an approximation to the ground-truth preference scores. The algorithm described above adopts a power method to compute the stationary distribution. Power methods are known to be computationally efficient in obtaining the leading eigenvalue of a sparse matrix \citep{Meirovitch97}. Initially starting with the uniform distribution over $[n]$, the algorithm iteratively computes the following until convergence:
\begin{align} \label{powermethod}
\boldsymbol{p}^{(t)} =  \hat{\boldsymbol{P}}\boldsymbol{p}^{(t-1)},
\end{align}
where $\boldsymbol{p}^{(t)}$ is a vector that represents the distribution of a random walk at iteration $t$. When convergence is reached, the algorithm returns the indices of the $K$ largest components of the distribution, which are the top-$K$ ranked items.

\subsection{Proof Outline} \label{subsec:linfty}
To distinguish the top-$K$ items from the rest, the pointwise error of each item becomes a fundamental bottleneck for top-$K$ ranking. It will be impossible to separate the $K^{th}$ and $(K + 1)^{th}$ ranked items unless their score separation exceeds the aggregate error of the score estimates for the two items. Based on this observation, we focus on figuring out the maximal pointwise error $\|\hat{\boldsymbol{w}}-\boldsymbol{w}\|_{\infty}$.

For the sake of clear demonstration, we use our stronger result in Theorem~\ref{thm:ergraph} attained by extending the results in Theorems~\ref{thm:generalgraph} and~\ref{thm:generalconverse} to the Erd{\H o}s-R{\' e}nyi model. The explanations carry over to the general model of our interest just as well, illustrating the motivation and the needed steps toward proving Theorem~\ref{thm:generalgraph}. Now, let us see how the following $\ell_\infty$ norm bound on the pointwise error (derived under the Erd{\H o}s-R{\' e}nyi random comparison graphs model) plays a key role in top-$K$ identification:
\begin{align} \label{ieq:linf}
\frac{\|\hat{\boldsymbol{w}}-\boldsymbol{w}\|_{\infty}}{\|\boldsymbol{w}\|_{\infty}}
\lesssim \sqrt{\frac{\log n}{npL}},
\end{align}
given $p > c_4 \sqrt{\frac{\log n}{n}}$ and $L \geq \left \lceil c_5 \frac{\log n}{np} \right \rceil$ where $c_4$ and $c_5$ are some constants.

We assume $\|\boldsymbol{w}\|_{\infty} = w_{\rm max} = 1$ for ease of presentation. Suppose $\Delta_K = w_K - w_{K+1} \gtrsim \sqrt{\frac{\log n}{npL}}$, then
\begin{align}
\hat{w}_i - \hat{w}_j & \geq w_i - w_j - | \hat{w}_i - w_i | - | \hat{w}_j - w_j | \geq w_K - w_{K+1} - 2 \| \hat{\boldsymbol{w}} - \boldsymbol{w} \|_{\infty} > 0,
\end{align}
for all $1 \leq i \leq K$ and $j \geq K+1$, indicating that the algorithm will output the top-$K$ items as desired. Hence, as long as $\Delta_K \gtrsim \sqrt{ \frac{ \log n}{ npL} }$ holds (coinciding with the claimed bound in (\ref{ieq:linf})), in other words, $pL \gtrsim \frac{\log n}{n \Delta_K^2}$ holds, reliable top-$K$ ranking is guaranteed with the sample size $\frac{n^2 pL}{2} \gtrsim \frac{n \log n}{\Delta_K^2}$.



What remains toward proving Theorem~\ref{thm:generalgraph} is the proof of the following, which is an $\ell_\infty$ norm bound on the maximal pointwise error in general comparison graphs (i.e., a generalized version of (\ref{ieq:linf})):
\begin{align} \label{ieq:linfgeneral}
\frac{\|\hat{\boldsymbol{w}}-\boldsymbol{w}\|_{\infty}}{\|\boldsymbol{w}\|_{\infty}}
\lesssim \left(1 + \frac{\sqrt{n}d_{\max}}{\gamma d_{\min}}\|\mathcal{L}^2\|_{2,\infty}\right)\sqrt{\frac{\log n}{Ld_{\max}}},
\end{align}
given $L \geq \left \lceil c_{1} \frac{\log n}{d_{\max}}\left(\frac{d_{\max}}{\gamma d_{\min}}\right)^2 \right \rceil$ where $c_1$ is some constant.

To prove (\ref{ieq:linfgeneral}), we first derive an upper bound (which we will prove at the end of this section) on the pointwise error between the score estimate of item $i$ at iteration $t$ and the true score, which consists of three terms:
\begin{align} \label{startpoint}
|p_{i}^{(t)} - w_i| 
\leq & | p_{i}^{(t-1)} - w_i | \hat{P}_{ii} + \sum_{j:j \neq i} | p_{j}^{(t-1)} - w_j | \hat{P}_{ij} + \Bigg| \sum_{j:j \neq i} (w_i + w_j) \left( \hat{P}_{ji} - P_{ji} \right) \Bigg|.
\end{align}

Then, we use the three lemmas stated below (which we prove in the following sections). We consider the regime where $n$ is sufficiently large. For $L \geq \left\lceil c_{1} \frac{\log n}{d_{\rm max}}\left(\frac{d_{\rm max}}{\gamma d_{\rm min}}\right)^2 \right\rceil$, applying Lemmas~\ref{lemma1},~\ref{lemma2}~and~\ref{lemma3} to (\ref{startpoint}) and solving it, we get
\begin{align}
\left| p_{i}^{(t)} - w_{i} \right|
\leq \lambda^{t} \left| p_{i}^{(0)} - w_{i} \right| + \left(c_{2}w_{\max} + c_{3}w_{\max}\frac{\sqrt{n}d_{\rm max}}{\gamma d_{\rm min}} \|\mathcal{L}^2\|_{2, \infty} \right) \sqrt{\frac{\log n}{Ld_{\max}}} + \epsilon_t,
\end{align}
where $\lambda < 1$, $c_{1} := \max \left( c_7, 4(1+b)^2\gamma^2  \right)$, $c_{2} := \frac{c_{8} + 2}{1 - \lambda}$, $c_{3} := \frac{c_{9}}{1 - \lambda}$ and $\epsilon_t > 0$ is a term that vanishes as $t$ tends to infinity. The above bound converges to $w_{\max}\left( c_2 + c_3\frac{\sqrt{n}d_{\rm max}}{\gamma d_{\rm min}} \|\mathcal{L}^2\|_{2, \infty} \right) \sqrt{\frac{\log n}{Ld_{\rm max}}} $ as $t$ tends to infinity. Since it holds for all $i$, we complete the proof of (\ref{ieq:linfgeneral}).

\begin{lemma} \label{lemma1}
For a comparison graph $\mathcal{G} = ([n], \mathcal{E})$,
\begin{align}
\Bigg|\sum\limits_{j : j \neq i} (w_{i} + w_{j}) \left(\hat{P}_{ji} - P_{ji} \right)\Bigg|
\leq 2w_{\mathrm{max}}\sqrt{\frac{\log n}{Ld_{\rm max}}}
\end{align}
with probability at least $1 - 2n^{-2}$.
\end{lemma}

\begin{lemma} \label{lemma2}
Suppose $L \geq 4(1+b)^2 \frac{\log n}{d_{\rm max}}\left(\frac{d_{\rm max}}{d_{\rm min}}\right)^2$, where $b := \frac{w_{\rm max}}{w_{\rm min}}$. Then,
\begin{align}
\hat{P}_{ii} < 1
\end{align}
with probability at least $1 - 2n^{- 2}$.
\end{lemma}

\begin{lemma} \label{lemma3}
Suppose $L \geq c_7 \frac{\log n}{d_{\rm max}}\left(\frac{d_{\rm max}}{\gamma d_{\rm min}}\right)^2$. Then, in the regime where $n$ is sufficiently large,
\begin{align}
\sum\limits_{j : j \neq i} |p_{j}^{(t)} - w_{j}|\hat{P}_{ij} 
\leq w_{\max}\left(c_{8} + c_{9}\frac{\sqrt{n}d_{\rm max}}{\gamma d_{\rm min}} \|\mathcal{L}^2\|_{2, \infty}\right) \sqrt{\frac{\log n}{Ld_{\rm max}}} + 6 w_{\rm max} \sqrt{b} t c_{10}^t
\end{align}
with probability at least $1 - 2n^{- 2}$, where $c_{7}$, $c_{8}$, $c_{9}$ and $c_{10} < 1$ are some constants.
\end{lemma}

We prove (\ref{startpoint}) here, and the proofs of the three lemmas are to follow.

\noindent \textbf{Proof of (\ref{startpoint}):} For fixed $i$, applying $\boldsymbol{p}^{(t)} =  \hat{\boldsymbol{P}}\boldsymbol{p}^{(t-1)}$, we get
\begin{align} \label{startpoint_empiricalw}
p_{i}^{(t)} = p_{i}^{(t-1)} \hat{P}_{ii} + \sum\limits_{j : j \neq i} p_{j}^{(t-1)}\hat{P}_{ij}.
\end{align}
Using the fact that random walks on an ideal version of matrix $\boldsymbol{\hat{P}}$ (matrix $\boldsymbol{P}$) are reversible, we get
\begin{align} \label{startpoint_idealw}
w_{i} & = w_i \Bigg( 1 - \sum\limits_{j : j \neq i} P_{ji} \Bigg) + w_i \sum\limits_{j : j \neq i} P_{ji} = w_i \Bigg(1 - \sum\limits_{j : j \neq i} P_{ji} \Bigg) + \sum\limits_{j : j \neq i} w_{j}P_{ij} \nonumber \\
& = \Bigg\{ w_i\hat{P}_{ii} + \sum_{j:j \neq i} w_i \left( \hat{P}_{ji} - P_{ji} \right) \Bigg\} + \Bigg\{ \sum_{j: j \neq i} w_j \hat{P}_{ij} - \sum_{j: j \neq i} w_j \left( \hat{P}_{ij} - P_{ij} \right)  \Bigg\}.
\end{align}
Using (\ref{startpoint_empiricalw}) and (\ref{startpoint_idealw}), we get
\begin{align} \label{startpoint_difference}
p_{i}^{(t)} - w_{i} = \left( p_{i}^{(t-1)} - w_{i} \right)\hat{P}_{ii} -  \sum_{j : j \neq i} w_{i} \left( \hat{P}_{ji} - P_{ji} \right) + \sum_{j : j \neq i} \left( p_{j}^{(t-1)} - w_{j} \right)\hat{P}_{ij} + \sum\limits_{j : j \neq i} w_{j} \left(\hat{P}_{ij} - P_{ij} \right).
\end{align}
We note that $\hat{P}_{ji} = \frac{1}{d_{\rm max}} - \hat{P}_{ij}$ from $y_{ji} = 1 - y_{ij}$. Similarly, $P_{ji} = \frac{1}{d_{\rm max}} - P_{ij}$. Thus, $\hat{P}_{ji} - P_{ji} = -\left(\hat{P}_{ij} - P_{ij}\right)$. Applying this equality and the triangle inequality to (\ref{startpoint_difference}), we get the recursive relation (\ref{startpoint}).

\subsection{Proof of Lemma~\ref{lemma1}} \label{sec:prooflemma1}
From the definitions of $\hat{P}_{ji}$ and $P_{ji}$,
\begin{align} \label{prooflemma1_startpoint}
& \Bigg| \sum_{j:j \neq i} \left( w_i + w_j \right) \left( \hat{P}_{ji} - P_{ji} \right) \Bigg|
= \frac{1}{L d_{\rm max}} \Bigg| \sum_{j:j \neq i}  \sum_{\ell=1}^{L} \left( \left( w_i + w_j \right) y_{ji}^{(\ell)} - w_j \right) \mathbb{I}\left[ (i,j) \in \mathcal{E} \right] \Bigg|.
\end{align}
First, let us bound the absolute value of the summations in (\ref{prooflemma1_startpoint}). Under the model of our interest, all pairwise comparison samples $y_{ji}^{(\ell)}$'s are independent over $(i,j)$ pair and $\ell$. Applying the Hoeffding inequality, conditional on $\mathcal{G} \sim \mathcal{G}_{n, p}$, we get
\begin{align}
\mathbb{P} \left[ \Bigg| \sum_{j:j \neq i} \sum_{\ell=1}^{L} \left( \left( w_i + w_j \right) y_{ji}^{(\ell)} - w_j \right) \mathbb{I}\left[ (i,j) \in \mathcal{E} \right] \Bigg| > t ~\middle|~ \mathcal{G} \right] \leq 2 \exp \left( -\frac{2t^2}{L d_i (2w_{\max})^2} \right).
\end{align}
Then we choose $t = 2w_{\max}\sqrt{  L d_i \log n }$, to get the tail probability as follows:
\begin{align}
2 \exp \left( -\frac{2 \left( 2w_{\max}\sqrt{ L d_i \log n } \right)^2}{L d_i (2w_{\max})^2} \right) = 2n^{-2}.
\end{align}
Therefore, with probability at least $1 - 2n^{-2}$,
\begin{align} \label{prooflemma1_abs}
\Bigg| \sum_{j:j \neq i} \sum_{\ell=1}^{L} \left( \left( w_i + w_j \right) y_{ji}^{(\ell)} - w_j \right) \mathbb{I}\left[ (i,j) \in \mathcal{E} \right] \Bigg| \leq 2w_{\max} \sqrt{  L d_i \log n }.
\end{align}
Now, let us put (\ref{prooflemma1_abs}) into (\ref{prooflemma1_startpoint}) and use $d_i \leq d_{\rm max}$. Conditional on $\mathcal{G} \sim \mathcal{G}_{n,p}$, with probability at least $1 - 2n^{-2}$, 
\begin{align} \label{prooflemma1_hoeffding}
& \Bigg| \sum_{j:j \neq i} \left( w_i + w_j \right) \left( \hat{P}_{ji} - P_{ji} \right) \Bigg| \leq 2w_{\max} \sqrt{ \frac{\log n}{ L d_{\rm max}} }.
\end{align}


\subsection{Proof of Lemma~\ref{lemma2}}
\label{sec:prooflemma2}
Using the Hoeffding inequality, as in the proof of Lemma~\ref{lemma1}, one can easily verify that, with probability at least $1 - 2n^{- 2}$,
\begin{align} \label{prooflemma2_bound}
\Bigg| \sum_{j:j \neq i} \left( \hat{P}_{ji} - P_{ji} \right) \Bigg| \leq \sqrt{\frac{\log n}{Ld_{\rm max}}}.
\end{align}
Using (\ref{prooflemma2_bound}), we get
\begin{align} \label{prooflemma2_ine1}
\hat{P}_{ii} = 1 - \sum_{j:j \neq i}\hat{P}_{ji} \leq 1 - \sum_{j:j \neq i} P_{ji} + \sqrt{\frac{\log n}{Ld_{\rm max}}}.
\end{align}
We let $b = \frac{w_{\rm max}}{w_{\rm min}}$. From the definition of $P_{ji}$, 
\begin{align} \label{prooflemma2_ine2}
\sum_{j:j \neq i} P_{ji} = \sum_{j:j \neq i} \frac{1}{d_{\rm max}} \frac{1}{1 + \frac{w_i}{w_j}} \mathbb{I}\left[ (i,j) \in \mathcal{E} \right] \geq \sum_{j:j \neq i} \frac{1}{d_{\rm max}} \frac{1}{1 + b} \mathbb{I}\left[ (i,j) \in \mathcal{E} \right] = \frac{d_i}{d_{\rm max}} \frac{1}{1 + b} \geq \frac{d_{\rm min}}{d_{\rm max}}\frac{1}{1+b}.
\end{align}
Putting (\ref{prooflemma2_ine2}) into (\ref{prooflemma2_ine1}), we get
\begin{align} \label{prooflemma2_ine3}
\hat{P}_{ii} \leq 1 - \frac{d_{\rm min}}{d_{\rm max}}\frac{1}{1 + b} + \sqrt{\frac{\log n}{Ld_{\rm max}}}.
\end{align}
Choosing $L = 4(1+b)^2\frac{\log n}{d_{\rm max}}(\frac{d_{\rm max}}{d_{\rm min}})^2$, we complete the proof of Lemma~\ref{lemma2}.


\subsection{Proof of Lemma~\ref{lemma3}}
\label{sec:prooflemma3}
We define a sequence as follows.
\begin{align} \label{prooflemma3_seqdef}
A^{(t)} := \sum_{j:j \neq i} \left| p_j^{(t)} - w_j \right| \hat{P}_{ij}.
\end{align}
From Lemma~\ref{lemma1}, with probability at least $1 - 2n^{-2}$,
\begin{align} \label{prooflemma3_const}
\Bigg| \sum_{j:j \neq i} \left( w_i + w_j \right) \left( \hat{P}_{ji} - P_{ji} \right) \Bigg| \leq 2 w_{\rm max} \sqrt{ \frac{\log n}{Ld_{\rm max}} }.
\end{align}
Putting (\ref{startpoint}) with (\ref{prooflemma3_const}) into (\ref{prooflemma3_seqdef}), we get
\begin{align} \label{prooflemma3_seqbound}
A^{(t)} \leq \sum_{j:j \neq i} \left| p_j^{(t-1)} - w_j \right| \hat{P}_{jj} \hat{P}_{ij} + 2w_{\rm max} \sqrt{\frac{\log n}{Ld_{\rm max}}} \sum_{j:j \neq i} \hat{P}_{ij} + \sum_{j:j \neq i} \sum_{k: k \neq j} \left| p_k^{(t-1)} - w_k \right|\hat{P}_{jk} \hat{P}_{ij}.
\end{align}
We simplify the last two terms. The first of the two is straightforward. The definition of $\hat{P}_{ij}$ gives $\sum_{j:j \neq i} \hat{P}_{ij} \leq 1$. The last term needs an extra effort. We defer the proof to a later part of this section, stating the following for now.
\begin{align} \label{prooflemma3_claim4result}
\sum_{j:j \neq i} \sum_{k: k \neq j}
\left| p_k^{(t-1)} - w_k \right| \hat{P}_{jk} \hat{P}_{ij} \leq \left\| \boldsymbol{p}^{(t-1)} - \boldsymbol{w} \right\|_2 \|\mathcal{L}^2\|_{2,\infty}.
\end{align}
Putting $\sum_{j:j \neq i} \hat{P}_{ij} \leq 1$ and (\ref{prooflemma3_claim4result}) into (\ref{prooflemma3_seqbound}), we get
\begin{align}
A^{(t)} \leq \sum_{j:j \neq i} \left| p_j^{(t-1)} - w_j \right| \hat{P}_{jj} \hat{P}_{ij} + 2w_{\rm max} \sqrt{\frac{\log n}{Ld_{\rm max}}} + \left\| \boldsymbol{p}^{(t-1)} - \boldsymbol{w} \right\|_2 \|\mathcal{L}^2\|_{2,\infty}.
\end{align}
From Lemma~\ref{lemma2}, we can find a constant $\beta$ such that $\hat{P}_{jj} \leq \beta < 1$ for all $j$. Using such $\beta$, we get
\begin{align} \label{prooflemma3_recursive}
A^{(t)} \nonumber \leq \beta A^{(t-1)}+ 2w_{\rm max} \sqrt{\frac{\log n}{Ld_{\rm max}}} + \left\| \boldsymbol{p}^{(t-1)} - \boldsymbol{w} \right\|_2 \|\mathcal{L}^2\|_{2,\infty}.
\end{align}
We now use an upper bound that prior work derived on $\left\| \boldsymbol{p}^{(t)} - \boldsymbol{w} \right\|_2$ (see Lemma 2 of \citep{Negahban12}). When $L \geq c_7 \frac{\log n}{d_{max}}\left(\frac{d_{\rm max}}{\gamma d_{\rm min}}\right)^2$, for some constants $c_{11} < 1$ and $c_{12} > 0$,
\begin{align}
\frac{\left\| \boldsymbol{p}^{(t)} - \boldsymbol{w} \right\|_2}{\left\| \boldsymbol{w} \right\|_2} \leq \sqrt{b} c_{11}^t \frac{\left\| \boldsymbol{p}^{(0)} - \boldsymbol{w} \right\|_2}{\left\| \boldsymbol{w} \right\|_2} + \frac{c_{12}}{\gamma d_{\rm min}} \sqrt{\frac{d_{\rm max} \log n}{L}}.
\end{align}
We use $\left\| \boldsymbol{w} \right\|_2 \leq \sqrt{n} \left\| \boldsymbol{w} \right\|_{\infty} = \sqrt{n} w_{\rm max}$ and $ \left\| \boldsymbol{p}^{(0)} - \boldsymbol{w} \right\|_2 \leq \sqrt{n} \left\| \boldsymbol{p}^{(0)} - \boldsymbol{w} \right\|_{\infty} \leq \sqrt{n} w_{\rm max}$ for uniformly distributed $\boldsymbol{p}^{(0)}$, to get
\begin{align} \label{prooflemma3_RCl2norm}
\left\| \boldsymbol{p}^{(t)} - \boldsymbol{w} \right\|_2 \leq \sqrt{n} w_{\rm max} \left( \sqrt{b} c_{11}^t  + \frac{c_{12}}{\gamma d_{\rm min}} \sqrt{\frac{d_{\rm max} \log n}{L}} \right).
\end{align}
We let $c_{8} := \frac{2}{1-\beta}$, $c_{9} := \frac{c_{12}}{1-\beta} $ and $c_{10} := \max (c_{11}, \beta) < 1$. Putting (\ref{prooflemma3_RCl2norm}) into (\ref{prooflemma3_recursive}) and solving it, we get
\begin{align}
A^{(t)} \leq w_{\max}\left(c_{8} + c_{9}\frac{\sqrt{n} d_{\rm max}}{\gamma d_{\rm min}}\|\mathcal{L}^2\|_{2, \infty}\right)\sqrt{\frac{\log n}{Ld_{max}}} + 7 w_{\rm max} \sqrt{b} t c_{10}^t.
\end{align}
From the definition of $A^{(t)}$, we complete the proof of Lemma~\ref{lemma3}.

\noindent \textbf{Proof of (\ref{prooflemma3_claim4result})}: 
By changing the order of the summations and the Cauchy-Schwarz inequality, we get
\begin{align} \label{proofclaim1_startpoint}
\sum_{j:j \neq i} \sum_{k: k \neq j}
\left| p_k^{(t)} - w_k \right| \hat{P}_{jk} \hat{P}_{ij} & = \sum_{k} \left| p_k^{(t)} - w_k \right| \sum_{j: j \notin \{i, k\}} \hat{P}_{jk} \hat{P}_{ij} \nonumber \\
& \leq \left\| \boldsymbol{p}^{(t)} - \boldsymbol{w} \right\|_2 \sqrt{ \sum_{k} \Bigg( \sum_{j: j \notin \{i, k\}} \hat{P}_{jk} \hat{P}_{ij} \Bigg)^2 }.
\end{align}
From the definitions of $\hat{P}_{jk}$ $\hat{P}_{ij}$, we can bound the term  $\sum_{j: j \notin \{i, k\}} \hat{P}_{jk} \hat{P}_{ij}$ as follows.
\begin{align} \label{proofclaim1_case1}
\sum_{j: j \notin \{i, k\}} \hat{P}_{jk} \hat{P}_{ij} 
& \leq \frac{1}{d_{\rm max}^2} \sum_{j: j \notin \{i, k\}} \mathbb{I}[(j,k) \in \mathcal{E}] \mathbb{I}[(i,j) \in \mathcal{E}] \leq \sum_{j: j \notin \{i, k\}} \frac{1}{d_{k}}\mathbb{I}[(k,j) \in \mathcal{E}] \frac{1}{d_{j}} \mathbb{I}[(j,i) \in \mathcal{E}] \nonumber \\
& = \sum_{j: j \notin \{i, k\}} \mathcal{L}_{kj}\mathcal{L}_{ji} = \sum_{j: j \notin \{i, k\}} [\mathcal{L}^2]_{ki},
\end{align}
where the inequality comes from $y_{jk} \leq 1$ and $y_{ij} \leq 1$. $\mathcal{L}$ is the Laplacian matrix whose entries are defined as $\mathcal{L}_{ij} := \frac{1}{d_{i}} \mathbb{I}[(i,j) \in \mathcal{E}]$.
Putting (\ref{proofclaim1_case1}) into (\ref{proofclaim1_startpoint}), we get
\begin{align}
\sum_{j:j \neq i} \sum_{k: k \neq j}
\left| p_k^{(t)} - w_k \right| \hat{P}_{jk} \hat{P}_{ij} & \leq \left\| \boldsymbol{p}^{(t)} - \boldsymbol{w} \right\|_2 \sqrt{ \sum_{k} \left( [ \mathcal{L}^2 ]_{ki} \right)^2 } \nonumber \\
& \leq \left\| \boldsymbol{p}^{(t)} - \boldsymbol{w} \right\|_2 \max_{i} \left\{ \sqrt{ \sum_{k} \left( [\mathcal{L}^2]_{ki} \right)^2 } \right\}
= \left\| \boldsymbol{p}^{(t)} - \boldsymbol{w} \right\|_2 \|\mathcal{L}^2\|_{2,\infty}.
\end{align}

\section{Proof of Theorem~\ref{thm:generalconverse}}\label{sec:generalconverse}
Theorem \ref{thm:generalconverse} establishes a lower bound of $S_{\Delta_K}$ in the minimax scenario. The proof in this section is modified from the proof of Theorem 2 in \citep{Chen15} to make the arguments therein hold valid in the general deterministic model of our interest. Similarly as in \citep{Chen15}, we intend to bound the minimax probability of error to characterize the conditions under which the probability cannot be made arbitrarily close to zero. We construct a finite set of hypotheses $\mathcal{M}$ and carry out an analysis based on classical Fano-type arguments. Each hypothesis is represented by a permutation $\sigma_{m} \in \mathcal{M}$ over $[n]$ and we denote by $\sigma_{m}(i)$ and $\sigma_{m}([K])$ the index of the $i^{th}$ ranked item and the index set of all top-$K$ items respectively.

We choose a set of hypotheses and some prior to be imposed on them. Suppose that the values of $\boldsymbol{w}$ are fixed up to permutation in such a way that
\begin{align} \label{hypothesis}
\forall \sigma_{m} \in \mathcal{M},  \text{  } w_{\sigma_{m}(i)} = \left\{
\begin{array}{ll}
    w_{K} & \text{if } 1 \leq i \leq K \\
	w_{K+1} & \text{if } K < i \leq n,
\end{array} \right.
\end{align}   
where we abuse the notation $w_{K}$, $w_{K+1}$ to represent any two values satisfying
\begin{align}
\frac{w_{K}-w_{K+1}}{w_{\max}} = \Delta_K >0.
\end{align}

Additionally, we impose a uniform prior over a collection $\mathcal{M}$ of $M:=\max(K, n-K)+1$ hypotheses regarding the permutation: if $K < \frac{n}{2}$, then
\begin{align}
& \forall \sigma_{m} \in \mathcal{M}, ~ \mathbb{P}\left[\sigma_{m}\right] = \frac{1}{M}, ~\sigma_{m}\left([K]\right) = \mathcal{S}_{m}, \text{ for }\mathcal{S}_{m}=\{2, ... ,K\} \cup \{m\}, ~(m=1,K+1,...,n), \label{lessthanhalf}
\end{align}
and if $K \geq \frac{n}{2}$, then
\begin{align}
& \forall \sigma_{m} \in \mathcal{M}, ~ \mathbb{P}\left[\sigma_{m}\right] = \frac{1}{M}, ~\sigma_{m}\left([K]\right) = \mathcal{S}_{m}, \text{ for }\mathcal{S}_{m}=\{1, ... ,K+1\} \backslash \{m\}, ~(m=1,...,K+1). \label{greaterthanhalf}
\end{align}
In words, each alternative hypothesis is made by interchanging two indices of the hypothesis complying to $\sigma([K])=[K]$. Denoting by $P_{e,M}$ the average probability of error with respect to the constructed prior, one can verify the minimax probability of error to be at least $P_{e,M}$.

Below is where we begin to modify the arguments in \citep{Chen15} for the general deterministic model of our interest. The modifications are mainly about adapting expressions that depend on the structure of comparison graphs such as the out-degrees of vertices. Random graphs constructed by the Erd{\H o}s-R{\'e}nyi model are concerned in \citep{Chen15}, whereas we consider deterministic graphs in this paper.

We bound the Bayesian probability of error using classical Fano-type bounds. To take partial observations into account, we introduce an erased version of $\boldsymbol{y}_{ij} := (y_{ij}^{(1)}, ... ,y_{ij}^{(L)})$ such that
\begin{align} \label{observation}
\boldsymbol{z}_{ij} = \left\{
\begin{array}{ll}
    \boldsymbol{y}_{ij} & \text{if } (i,j) \in \mathcal{E} \\
	\text{erasure} & \text{else},
\end{array} \right.
\end{align}
and set $\boldsymbol{Z} := \{\boldsymbol{z}_{ij}\}_{1\leq i < j \leq n}$. We apply the generalized Fano inequality \citep{Han94} to get
\begin{align} \label{Fano}
P_{e,M}
\geq 1 - \frac{1}{\log M} \left\{ \frac{1}{M^2} \sum\limits_{\sigma_{a},\sigma_{b} \in \mathcal{M}} D(\mathbb{P}_{\boldsymbol{Z}|\sigma=\sigma_{a}} || \mathbb{P}_{\boldsymbol{Z}|\sigma=\sigma_{b}}) + \log 2 \right\},
\end{align}
where $D(P||Q)$ denotes the Kullback-Leibler (KL) divergence of $Q$ from $P$. The proof will be finished once we show that (\ref{Fano}) can be further bounded by some positive constant. To that end, first we get, by (\ref{observation}) and the independence assumption of $y_{ij}^{(\ell)}$,
\begin{align} \label{KL expansion}
D \left(\mathbb{P}_{\boldsymbol{Z}|\sigma=\sigma_{a}} || \mathbb{P}_{\boldsymbol{Z}|\sigma=\sigma_{b}} \right) 
& = L\sum\limits_{(i,j):(i,j) \in \mathcal{E}} D\left(\mathbb{P}_{y_{ij}^{(1)}|\sigma=\sigma_{a}} || \mathbb{P}_{y_{ij}^{(1)}|\sigma=\sigma_{b}}\right),
\end{align}
where the derivation follows the same line of arguments in \citep{Chen15}.

We now intend to find an upper bound on (\ref{KL expansion}). Note the following difference between two hypotheses $\sigma_{a}$ and $\sigma_{b}$ when $K < n/2$: according to the definitions given in (\ref{hypothesis}) and (\ref{lessthanhalf}), $w_{a} = w_{K}$ and $w_{b} = w_{K+1}$ for hypothesis $\sigma_{a}$, and $w_{a} = w_{K+1}$ and $w_{b} = w_{K}$ for hypothesis $\sigma_{b}$. Hence, we can get $D\left(\mathbb{P}_{y_{ij}^{(1)}|\sigma=\sigma_{a}} || \mathbb{P}_{y_{ij}^{(1)}|\sigma=\sigma_{b}}\right) = 0$ if both $i$ and $j$ are neither $a$ nor $b$. In other words, as long as we are concerned with pairwise observations that involve neither $a$ nor $b$, the distributions based on two hypotheses $\sigma_{a}$ and $\sigma_{b}$ are identical. We can get the same result when $K \geq n/2$ in a similar way. Using the upper bound $D\left(\mathbb{P}_{y_{ij}^{(1)}|\sigma=\sigma_{a}} || \mathbb{P}_{y_{ij}^{(1)}|\sigma=\sigma_{b}}\right) \leq \frac{w_{\max}^4}{w_{\min}^4}\Delta_K^2$ derived in Lemma 3 of \citep{Chen15},
\begin{align} \label{KL bound}
\sum\limits_{(i,j):(i,j) \in \mathcal{E}}  D\left(\mathbb{P}_{y_{ij}^{(1)}|\sigma=\sigma_{a}} || \mathbb{P}_{y_{ij}^{(l)}|\sigma=\sigma_{b}}\right)
\leq (d_{a} + d_{b}) \frac{w_{\max}^4}{w_{\min}^4}\Delta_K^2,
\end{align}
since in summing over all edges, there are $d_a + d_b - 1$ edges in total that result in non-zero KL divergences.
Putting (\ref{KL expansion}) and (\ref{KL bound}) into (\ref{Fano}), we finally get a lower bound on $P_{e, M}$.
\begin{align}
P_{e,M} 
& \geq 1- \frac{1}{\log M}\left\{\frac{L}{M^2} \sum\limits_{\sigma_{a} \in \mathcal{M}} \sum\limits_{\sigma_{b} \in \mathcal{M}} \left( d_{a} + d_{b} \right)\frac{w_{\max}^4}{w_{\min}^4}\Delta_K^2 + \log 2 \right\} \nonumber \\
& = 1- \frac{1}{\log M}\left\{\frac{L}{M^2} \frac{w_{\max}^4}{w_{\min}^4}\Delta_K^2 \left( \sum\limits_{\sigma_{b} \in \mathcal{M}}\sum\limits_{\sigma_{a} \in \mathcal{M}} d_{a}  + \sum\limits_{\sigma_{a} \in \mathcal{M}} \sum\limits_{\sigma_{b} \in \mathcal{M}} d_{b}  \right) +  \log 2 \right\} \nonumber \\
& = 1- \frac{1}{\log M}\left\{\frac{L}{M^2} \frac{w_{\max}^4}{w_{\min}^4}\Delta_K^2 \left( 2M\sum\limits_{\sigma_{a} \in \mathcal{M}} d_{a} \right) + \log 2 \right\} \nonumber \\
& \geq 1 - \frac{1}{\log M}\left\{\frac{2L}{M}\frac{w_{\max}^4}{w_{\min}^4}\Delta_K^2 \sum\limits_{i} d_{i} + \log 2 \right\},
\end{align}
where the last inequality follows by the fact that $\sum_{i} d_i \geq \sum_{\sigma_{a} \in \mathcal{M}} d_a$.

One would have $P_{e} \geq P_{e, M} \geq \epsilon$ for fixed $\epsilon 
\in (0, \frac{1}{2})$, if
\begin{align}
\frac{8w_{\max}^4}{nw_{\min}^4}L\Delta_K^2 |\mathcal{E}| + \log 2
\leq (1-\epsilon)\log n \Longleftrightarrow L|\mathcal{E}| \leq \frac{w_{\min}^4}{8w_{\max}^4}\frac{n((1-\epsilon)\log n - \log 2)}{\Delta_K^2},
\end{align}
since $M \geq \frac{n}{2}$ and $|\mathcal{E}| = \frac{\sum_{i}d_{i}}{2}$. This completes the proof of Theorem~\ref{thm:generalconverse}.

\section{Proof of Theorem~\ref{thm:ergraph}} \label{sec:ergraph_proof}
The proof of Theorem~\ref{thm:ergraph} comes down to showing 
\begin{align} \label{linfer}
\frac{\|w-\hat{w}\|_{\infty}}{\|w\|_{\infty}} 
\lesssim \sqrt{\frac{\log n}{npL}}.
\end{align}
Once it is shown, together with Theorem \ref{thm:generalconverse}, the optimal sample complexity (\ref{minsamplecomplexity}) is characterized. Hence, we mainly seek to show (\ref{linfer}) in this section.

By the Bernstein inequality in Lemma~\ref{concentration}, $\frac{1}{2}np \leq d_{i} \leq \frac{3}{2}np$ for all $i$. Given $p > \frac{\log n}{n}$, it is shown that $\gamma \geq \frac{1}{2}$ in Lemma 7 of \citep{Negahban12}. We now show that $\sqrt{n}\|\mathcal{L}^2\|_{2,\infty}$ is bounded by a constant.
\begin{align} \label{specialcase}
\sqrt{n}\left\| \mathcal{L}^2 \right\|_{2,\infty} & = \sqrt{n} \max_{i} \left\{ \sqrt{\sum\limits_{k} \left( \sum\limits_{j:j \notin \{i,k\}} \frac{1}{d_{k}}\mathbb{I}[(k,j) \in \mathcal{E}] \frac{1}{d_{j}}\mathbb{I}[(j,i) \in \mathcal{E}] \right)^2} \right\} \nonumber\\
& \leq \frac{\sqrt{n}}{d_{\rm min}^2} \max_{i}\left\{ \sqrt{\sum\limits_{k} \left( \sum\limits_{j:j \notin \{i,k\}} \mathbb{I}[(k,j) \in \mathcal{E}] \mathbb{I}[(j,i) \in \mathcal{E}] \right)^2 }
 \right\}.
\end{align}
Given a comparison graph $\mathcal{G} \sim \mathcal{G}_{n,p}$ where $p > c_4 \sqrt{\frac{\log n}{n}}$ and $k \neq i$, using the Bernstein inequality in Lemma~\ref{concentration}, we get
\begin{align}
\sum_{j: j \notin \{i, k\}} & \mathbb{I}[(i,j) \in \mathcal{E}] \mathbb{I}[(j,k) \in \mathcal{E}] \leq \frac{3}{2}np^2, \label{proofclaim1_indicators}
\end{align}
and
\begin{align}
\sum_{j: j \neq i} & \mathbb{I}[(i,j) \in \mathcal{E}] \mathbb{I}[(j,i) \in \mathcal{E}] = \sum_{j: j \neq i} \mathbb{I}[(i,j) \in \mathcal{E}] \leq \frac{3}{2}np, \label{proofclaim1_case2bound}
\end{align}
both with probability at least $1-2n^{-\frac{3}{28}c_4^2}$.\\
Putting (\ref{proofclaim1_indicators}) and (\ref{proofclaim1_case2bound}) into (\ref{specialcase}), we get
\begin{align} \label{laplaciannorm}
\sqrt{n}\left\| \mathcal{L}^2 \right\|_{2,\infty} \leq \frac{\sqrt{n}}{\left(\frac{1}{2}np\right)^2}\sqrt{(n-1)\frac{9}{4}n^2p^4 + \frac{9}{4}n^2p^2} \leq \frac{1}{\frac{1}{4}n^{\frac{3}{2}}p^2}\frac{3}{2}n^{\frac{3}{2}}p^2 \sqrt{1 + \frac{1}{np^2}} = 12
\end{align}
for large $n$ where $p > c_4 \sqrt{\frac{\log n}{n}}$.\\
Putting (\ref{laplaciannorm}), Lemma~\ref{concentration}, and Lemma 7 of \citep{Negahban12} into equation (\ref{ieq:linfgeneral}), we get
\begin{align} \label{linfer_2}
\frac{\|w-\hat{w}\|_{\infty}}{\|w\|_{\infty}} 
\lesssim \left(1 + \frac{\sqrt{n}d_{\rm max}}{\gamma d_{\rm min}} \|\mathcal{L}^2\|_{2, \infty}\right) \sqrt{\frac{\log n}{Ld_{\rm max}}} \leq 73 \sqrt{\frac{\log n}{L(\frac{1}{2}np)}} \lesssim \sqrt{\frac{\log n}{npL}}.
\end{align}
Therefore, for reliable ranking, we need the condition as follows,
\begin{align}
\frac{n^2pL}{2} \geq c_6\frac{n \log n}{\Delta_K^2}.
\end{align}
for some constant $c_6$. Similarly, using Lemma~\ref{concentration}, and Lemma 7 of \citep{Negahban12}, condition of $L$ becomes as follows, which completes the proof of Theorem \ref{thm:ergraph}.
\begin{align}
L \geq \left\lceil c_{1}\frac{\log n}{d_{\rm max}}\left(\frac{d_{\rm max}}{\gamma d_{\rm min}}\right)^2 \right\rceil
\geq \left\lceil 24c_{1}\frac{\log n}{np} \right\rceil
= \left\lceil c_{5}\frac{\log n}{np} \right\rceil.
\end{align}


%




\appendix
\section*{Appendix A. Concentration of Degrees of Items (Lemma \ref{concentration})}
\label{apx:concentration}
Although it is clear that the empirical mean of a random variable converges to its true mean with high probability, we prove the following for rigorous proofs in this paper.

\begin{lemma} \label{concentration}
Suppose independent and identically distributed (i.i.d.) random variables $X_i$ follow Bernoulli($q$) and $q > c \frac{\log n}{n}$. Then, with probability at least $1-2n^{-\frac{3}{28}c}$,
\begin{align}
\frac{1}{2}nq \leq \sum_{i=1}^{n} X_i \leq \frac{3}{2}nq.
\end{align}
\end{lemma}
\begin{proof}
Applying the Bernstein inequality, we get
\begin{align}
\mathbb{P} \left[ \left| \sum_{i=1}^{n}X_i - nq \right| > t \right]
& \leq 2 \exp \left( - \frac{\frac{1}{2}t^2}{ nq + \frac{1}{3} t } \right).
\end{align}
Then we choose $t = \frac{1}{2}nq$ and use $ q > c \frac{\log n}{n} $, to get the following tail probability, which completes the proof.
\begin{align}
\mathbb{P} \left[ \left| \sum_{i=1}^{n}X_i - nq \right| > \frac{1}{2}nq \right] \leq 2 n^{-\frac{3}{28} \frac{nq}{\log n} } < 2 n^{-\frac{3}{28}c}.
\end{align}
\end{proof}


\vskip 0.2in

\bibliography{JMLR_TopKDense}
\bibliographystyle{natbib}

\end{document}